%% file: main.tex
\documentclass[letter, 10 pt, journal, twoside]{IEEEtran} 

\usepackage[english]{babel}
\addto\captionsenglish{\renewcommand{\figurename}{Fig.}}
\usepackage[utf8]{inputenc}
\usepackage{times} 
\usepackage{cite}
\usepackage{booktabs}
\usepackage[hidelinks]{hyperref}
\usepackage[nolist]{acronym}
\usepackage{xurl}
\usepackage{multirow}

\usepackage{siunitx}
\usepackage{graphicx} 
\usepackage{epsfig} 
\usepackage{pgfplots}
\pgfplotsset{compat=1.18}
\pgfplotsset{every axis/.append style={grid=major, grid style={gray!25, line width=0.3pt}}}
\usepgfplotslibrary{statistics}
\usepgfplotslibrary{groupplots}

\usepackage{float}

\usepackage[export]{adjustbox}

\usepackage{xcolor}
\definecolor{myYellow}{rgb}{0.93,0.69,0.13}
\definecolor{myPurple}{rgb}{0.49,0.18,0.56}
\definecolor{myGreen}{rgb}{0.26 0.72 0.54}
\definecolor{darkgreen}{rgb}{0.272, 0.50, 0.376}
\definecolor{lightgreen}{rgb}{0.585, 0.82, 0.647}

\definecolor{cOrange}{RGB}{255,127,0}
\definecolor{cRed}{RGB}{214,39,40}
\definecolor{cBlue}{RGB}{31,119,180}
\definecolor{cGrey}{RGB}{127,127,127}
\definecolor{cGreen}{RGB}{44,160,44}

\colorlet{mydarkblue}{blue!30!black}

\usepackage{mathtools}
\usepackage{nicefrac}
\usepackage{amsmath}
\usepackage{amssymb}
\usepackage{bm} 
\usepackage{scalerel} 

\usepackage{amsthm}
\newtheoremstyle{compact}
  {4pt}   
  {4pt}   
  {\itshape}
  {}
  {\bfseries}
  {.}
  {0.5em}
  {}

\theoremstyle{compact}
\newtheorem{theorem}{Theorem}
\newtheorem{lemma}[theorem]{Lemma}
\newtheorem{proposition}[theorem]{Proposition}
\newtheorem{corollary}[theorem]{Corollary}

\newtheorem{remark}[theorem]{Remark}

\newcommand{\R}{\mathbb{R}}
\newcommand{\Cset}{\mathcal{C}}
\newcommand{\Ebar}{\mathcal{E}}
\newcommand{\abar}{\bar{a}}
\newcommand{\tbar}{\bar{\tau}}
\newcommand{\lbar}{\bar{\ell}}
\newcommand{\ldrop}{\ell^{\dagger}}
\newcommand{\vstar}{v^{\star}}
\newcommand{\tr}{\operatorname{tr}}
\newcommand{\diag}{\operatorname{diag}}
\newcommand{\rank}{\operatorname{rank}}
\newcommand{\Ker}{\operatorname{ker}}
\newcommand{\Range}{\operatorname{range}}
\newcommand{\Fw}{\mathcal{F}_w}

\DeclareMathAlphabet{\pazocal}{OMS}{zplm}{m}{n}
\def\xx{\mathsf{x}}

\usepackage{etoolbox}
\makeatletter%
\AfterPreamble{%
	\usepackage{hyperref}%
	\let\oldhypertarget\hypertarget%
	\renewcommand{\hypertarget}[2]{%
		\oldhypertarget{#1}{#2}%
		\protected@write\@mainaux{}{%
			\string\expandafter\string\gdef%
			\string\csname\string\detokenize{#1}\string\endcsname{#2}%
		}%
	}%
	\newcommand{\myhyperlink}[1]{%
		\hyperlink{#1}{\csname #1\endcsname}%
	}%
}
\makeatother%

\usepackage[ruled,vlined,linesnumbered]{algorithm2e}

\makeatletter
\def\BState{\State\hskip-\ALG@thistlm}
\makeatother

\usepackage{tikz}
\pgfdeclarelayer{foreground}
\pgfsetlayers{background, main, foreground}
\usetikzlibrary{quotes, angles, backgrounds, arrows, automata, shapes, positioning, calc, through, spy, decorations.pathreplacing, decorations.markings, arrows.meta, automata, petri, shapes.multipart, patterns}

\newlength{\panelwd}
\newlength{\panelht}

\tikzset{
    imglabel/.style={
      rectangle,
      inner sep=2pt,
      text=black,
      minimum height=1em,
      text centered,
      fill=white,
      fill opacity=1.0,
      text opacity=1,
      anchor=south west,
    },
  }
\tikzset{
	state/.style={
		rectangle,
		draw=black, very thick,
		minimum height=1.0em,
		text centered,
	},
}
\tikzset{
  on each segment/.style={
    decorate,
    decoration={
      show path construction,
      moveto code={},
      lineto code={
        \path [#1]
        (\tikzinputsegmentfirst) -- (\tikzinputsegmentlast);
      },
      curveto code={
        \path [#1] (\tikzinputsegmentfirst)
        .. controls
        (\tikzinputsegmentsupporta) and (\tikzinputsegmentsupportb)
        ..
        (\tikzinputsegmentlast);
      },
      closepath code={
        \path [#1]
        (\tikzinputsegmentfirst) -- (\tikzinputsegmentlast);
      },
    },
  },
  mid arrow/.style={postaction={decorate,decoration={
        markings,
        mark=at position .5 with {\arrow[#1]{stealth}}
      }}},
}

\tikzset{
  half circle/.style={
      semicircle,
      shape border rotate=180,
      anchor=chord center,
      minimum size=5mm
      }
}

\graphicspath{{./figures/}}

\newcommand{\coderepo}{{\url{https://github.com/ctu-mrs/readiness-barrier-functions}}}

\title{Readiness Barrier Functions: Forward-Invariant Control Authority for Overactuated Multirotor Allocation}

\author{Giuseppe Silano,~\IEEEmembership{Senior Member,~IEEE} 
    \thanks{Received 6 August 2026; revised 19 October 2026; accepted 17 November 2026. Date of publication Day Month Year; date of current version Day Month Year. This article was recommended for publication by Associate Editor N. Surname and Editor C. Soon-Jo upon evaluation of the reviewers' comments. This work was supported by the Czech Science Foundation (GAČR) within research project no. 26-22606S.} 
    \thanks{Giuseppe Silano is with the Department of Cybernetics, Faculty of Electrical Engineering, Czech Technical University in Prague, 12000 Prague, Czech Republic (e-mail: {\tt\small giuseppe.silano@fel.cvut.cz}).} 
    \thanks{The author used Google Gemini 4 and Anthropic Claude Opus 5.5 for proofreading and exposition refinement, but independently verified all outputs and assume full responsibility for the article's content.}
    \thanks{Digital Object Identifier (DOI): see top of this page.}
}

\begin{document}

\maketitle


\begin{acronym}
    \acro{CBF}[CBF]{Control Barrier Function}
    \acro{ESC}[ESC]{Electronic Speed Controller}
    \acro{MPC}[MPC]{Model Predictive Control}
    \acro{QP}[QP]{Quadratic Program}
    \acro{RMS}[RMS]{Root Mean Square}
\end{acronym}



\begin{abstract}
    Allocators that greedily maximize a readiness metric over the actuator fiber bundle of an overactuated multirotor jump between disconnected optimal strata, demanding actuator rates no motor can deliver; effort-minimizing allocators are continuous but cannot keep wrench-rate authority above any certified level. We reconcile the two by treating authority as a forward-invariant quantity: a control barrier function on the log-determinant of the drag-aware authority co-metric, enforced at torque level by a quadratic program in the allocation null space. A single design inequality makes the certified set compact and strictly interior to the actuator box; the readiness cost of a rotor deactivation is $\ln(n/(n{-}m))$ for symmetric designs. An alignment ratio states when tracking stays exact, otherwise the wrench error is $\mathcal{O}(\rho^{-1/2})$, and a robust variant handles motor-parameter uncertainty with a closed-form floor shift independent of the airframe matrix. On a hexarotor and a fully-actuated octorotor the closed-form gap matches simulation to machine precision. Over $300$ randomized missions, greedy maximization violates the certified floor in up to $51\%$ of authority-scarce missions, while the proposed filter never leaves the certified set and, in the hardest moment reversal, commits a $79\times$ smaller wrench error. The robust variant holds the floor on all $150$ perturbed plants. A high-fidelity simulation with measured propeller and motor data shows where the floor holds under unmodeled aerodynamics, electrical limits, and delay, and where it stops applying.
\end{abstract}






\section{Introduction}
\label{sec:introduction}

\IEEEPARstart{O}{veractuated} multirotors---tilted-\allowbreak propeller hexarotors, fully actuated octorotors, omnidirectional platforms---carry more rotors than the wrench dimensions they control \cite{OlleroTRO2022, Rashad2020RAM}. The redundancy enables fault tolerance, decoupled force-and-moment generation, and physical interaction, but at every instant a continuum of rotor-speed vectors produces the commanded wrench and the allocator must select one \cite{Johansen2013allocation, Bodson2002allocation}.

The classical answer minimizes actuator effort by constrained least-squares or linear programming \cite{Johansen2013allocation, Bodson2002allocation}, for standard and fully-actuated multirotors \cite{Ryll2015overactuated, Michieletto2018fundamental} and inside predictive controllers \cite{Bicego2020nmpc}. Authority enters only implicitly through the feasible set, and the cost is blind to \emph{wrench-rate authority}, the symmetric acceleration capacity to generate rapid wrench variations. Since thrust and drag both scale quadratically with spin \cite{Bicego2020nmpc, Bristeau2009ECC}, a slow rotor produces thrust with vanishing slope while a saturated rotor has spent its torque budget fighting drag; authority peaks at an interior \emph{sweet spot} and collapses at both ends. Stopping one rotor squashes the achievable wrench-rate set \emph{anisotropically}, saturating all rotors shrinks it \emph{isotropically} (Figure~\ref{fig:ellipsoid}); an effort-minimizing cost sees neither.

\begin{figure}[tb]
    \centering
    \resizebox{0.60\columnwidth}{!}{%
        \input{figures/tikz/fig_ellipse.tex}}
    \vspace{-0.55em}
    \caption{Achievable wrench-rate set $\mathcal{R}$: sweet spot (blue), rotor~1 stopped (orange dashed), near saturation (red dotted). The commanded $\dot w_{\mathrm{des}}$ must lie inside $\mathcal{R}$ to be realizable.} 
    \label{fig:ellipsoid}
\end{figure}

Determinant-based manipulability \cite{Yoshikawa1985manip} is a standard singularity-avoidance penalty. Recent work extends it to multirotors via a drag-aware authority co-metric---\emph{Drag-Aware Aerodynamic Manipulability} (DAAM)---whose log-determinant measures the remaining wrench-rate volume; maximizing it along the \emph{task fiber}, the set of redundant speeds that produce the commanded wrench, repels both collapse modes \cite{Franchi2026aeropromptness}. However, the fiberwise maximizer is set-valued, with optimal selections on disconnected sheets \cite{Franchi2026aeropromptness, Franchi2026muscle}.

The obstruction is structural. The signed-quadratic thrust map partitions rotor-speed space into sign orthants whose boundaries are the loci where a rotor reverses direction; there the thrust slope vanishes, the allocation Jacobian degenerates, and the actuator rate sustaining a bounded wrench rate diverges \cite{Franchi2026foliation}. Tracking a continuous task across orthants thus demands an acceleration no motor can deliver, and the saturated rotors inject the wrench error the allocator exists to null. Low-pass filtering \cite{Johansen2013allocation, Bodson2002allocation} bounds the rate demand but drags the system \emph{through} the low-authority crossing rather than around it (Section~\ref{sec:sim}). Authority is thus either a byproduct of effort minimization or an unconstrained objective, never a forward-invariant property. 

We therefore argue that control authority should be \emph{certified} rather than greedily maximized. We treat the readiness log-determinant as a \ac{CBF} \cite{Ames2017CBFQP}, grounded in set-invariance theory \cite{Blanchini1999setinvariance} and its robust and discrete-time extensions \cite{Xu2015robustness, Agrawal2017discrete}, and floor it at the torque level via a quadratic program acting in the allocation null space. Feasibility then turns on the geometry of the fiber tangent space, and the single inequality $h(v) \ge 0$ fences both failure modes. The certified set has $2^n$ components; enforcing the barrier confines the trajectory to its initial one and guarantees \emph{forward invariance} for the modeled dynamics. When tracking conflicts with the floor, a slack $\delta = \mathcal{O}(\rho^{-1/2})$ relaxes the wrench, giving a continuous allocation where a greedy maximizer must jump.

\textbf{Contributions.} \emph{(i) Readiness geometry:} a closed-form gradient and an exact component-separation barrier, the readiness lost when rotor $k$ deactivates being $-\ln(1-\lambda_k)$ for a capacity-weighted leverage $\lambda_k$ (Theorem~\ref{thm:gap}). \emph{(ii) Safe set and invariance:} one design inequality $\lbar>\ldrop$ makes the safe set compact and strictly interior to the actuator box (Proposition~\ref{prop:struct}), with forward invariance (Theorem~\ref{thm:inv}). \emph{(iii) Tracking:} an alignment ratio dictating when tracking stays exact and an $\mathcal{O}(\rho^{-1/2})$ wrench-error bound otherwise (Theorem~\ref{thm:feas}). \emph{(iv) Robustness:} a variant separating a worst-case metric bound from robust invariance (Theorem~\ref{thm:robust}) via a closed-form drift term, without which the guarantee fails on $48\%$ of plants at $30\%$ mismatch.
A slope corollary (Corollary~\ref{cor:slope}) states how airspeed enters the certificates, which are validated on a hexarotor and a tilted octorotor over randomized missions, parametric uncertainty, delay, and forward flight.



\section{Modeling and Notation}
\label{sec:model}

An overactuated multirotor with $n$ rotors produces a control wrench $w \in \R^m$ ($n>m$) via $f(v) = A\phi(v)$, where $v\in\R^n$ collects the signed spin rates---we assume bidirectional propulsion with torque-controlled rotors capable of reversal through $v_i=0$---$\phi(v) = v\odot|v|$ is the signed-quadratic thrust map, and $A\in\R^{m\times n}$ is the constant, full-row-rank allocation matrix with columns $A_{\bullet i}$. The \emph{allocation fiber} $\Fw := f^{-1}(w)$ is the continuum of redundant configurations producing $w$, from which the allocator selects one point. Per rotor, 
\begin{equation}\label{eq:motor}
    m_i\dot v_i=-b_iv_i|v_i|+\tau_i,\qquad |\tau_i|\le\tbar_i ,
\end{equation}
where $m_i$ is the motor-propeller inertia, $b_i>0$ the lumped drag coefficient, and $\tbar_i>0$ the torque limit. The quadratic form $b_iv_i|v_i|$ is the quasi-static rotor-drag model \cite{Bicego2020nmpc, Bristeau2009ECC}, valid at hover and low advance ratio; in fast forward flight, inflow modulates $b_i$, treated as bounded parametric variation in Section~\ref{sec:robust}, and the thrust slope, treated by Corollary~\ref{cor:slope}. Stacking \eqref{eq:motor} gives the vector dynamics $\dot v=d(v)+M^{-1}\tau$, with $M=\diag(m_i)$, and the drag vector $d(v)=-M^{-1}\diag(b)\phi(v)$. The \emph{Symmetric Acceleration Capacity} (SAC) is the largest torque-limited acceleration available symmetrically at spin rate $v_i$:
\begin{equation}
    \abar_i(v_i)=\tfrac{\tbar_i-b_iv_i^2}{m_i},\quad
    \Ebar=\{v: |v_i|<v_i^{\mathrm{sat}}:=\sqrt{\tbar_i/b_i}\ \forall i\},
\end{equation}
where $v_i^{\mathrm{sat}}$ is the \emph{saturation speed} and $\Ebar$ the open feasible box on which $\abar_i>0$. With allocation Jacobian $J(v)=2A\diag(|v|)$ and SAC weighting $W(v)=\diag(\abar_i^2)$, the deliverable wrench rates form the ellipsoid $\mathcal{R}(v)=\{JW^{1/2}u:\|u\|\le1\}$, whose Gram matrix is the \emph{drag-aware authority co-metric} \cite{Franchi2026aeropromptness}:
\begin{align}\label{eq:D}
    D(v)&=J(v)W(v)J(v)^\top=4A\,\Psi(v)\,A^\top,\\
    \psi_i(v_i)&=v_i^2\abar_i(v_i)^2 ,
\end{align}
where $\Psi=\diag(\psi_i)$. Equivalently $\mathcal{R}(v)=\{x:x^\top D^{-1}x\le1\}$, with semi-axes $\sqrt{\lambda_i(D)}$ and $m$-dimensional volume $\omega_m\sqrt{\det D}$ ($\omega_m$ the unit-ball volume). The \emph{readiness level}
\begin{equation}\label{eq:volume}
    L(v):=\ln\det D(v)=2\ln\bigl(\mathrm{vol}_m\,\mathcal{R}(v)/\omega_m\bigr).
\end{equation}
is therefore a log-volume of achievable wrench rate.

A floor on $L$ bounds the wrench-rate \emph{volume} and, since $\lambda_{\min}(D)\ge e^{\lbar}/\bar\Lambda^{m-1}$ (Proposition~\ref{prop:struct}(iv)), an authority radius along \emph{every} axis; unlike a trace it sees a collapsed axis, and unlike $\lambda_{\min}$ it keeps the others \cite{Yoshikawa1985manip, Franchi2026aeropromptness}.

We reserve $h$ for the \emph{shifted barrier}, where $\lbar\in(\ldrop,L^{\mathrm{op}}]$ is the chosen authority floor (specified in Section~\ref{sec:theory}):
\begin{equation}\label{eq:hdef}
    h(v)=L(v)-\lbar,\qquad \Cset=\{v\in\Ebar:h(v)\ge0\},
\end{equation}
$\Cset$ being the \emph{certified-authority set}. Throughout, \textbf{A1}: $\rank A=m$, $n>m$, and $A_{\bullet i}\neq0$ for every $i$.

The column condition is not implied by rank (zeroing a column of a $4\times6$ hexarotor matrix keeps rank 4), yet without it $\Delta_k=0$ and the window~\eqref{eq:window} collapses. Index $i$ is generic; $k$ marks a specific rotor.



\section{Theory and Allocation Architecture}
\label{sec:theory}

The rotor-speed space $\Ebar$ is partitioned by the hyperplanes $v_k=0$ into $2^n$ \emph{sign orthants}, within each of which $L$ is smooth with global maxima at the sweet spots $\pm\vstar$, $|v_i^\star|=v_i^\mathrm{sat}/\sqrt{3}$ (Lemma \ref{lem:grad}). The fiberwise $\arg\max L$ is therefore set-valued: a greedy allocator can undergo a discontinuous branch change when the optimal strata $\pm\vstar$ exchange dominance as the commanded wrench reverses. Any such change passes rotor $k$ through $v_k=0$, where its thrust slope vanishes, its column drops out of $D$, and the Jacobian loses rank; our first result quantifies the cost.

\begin{theorem}[Dropout Gap]\label{thm:gap}
The largest readiness level attainable with rotor $k$ inactive ($\psi_k=0$) is $L^{\max}-\Delta_k$, where
\begin{equation}\label{eq:gap}
    \begin{aligned}
    \Delta_k &= -\ln(1-\lambda_k)\in(0,+\infty],\\
    \lambda_k &= \psi_k^\star A_{\bullet k}^\top S^{-1}A_{\bullet k}\in(0,1],
    \end{aligned}
\end{equation}
with $\lambda_k$ the $\psi^\star$-weighted leverage of rotor $k$. Here, $L^\mathrm{max}:=\max_{\Ebar} L$ (Lemma \ref{lem:grad}) and $S := D(\vstar)/4$ is a quarter of the co-metric at the sweet spot, so $S=\sum_i\psi_i^\star A_{\bullet i}A_{\bullet i}^\top$ with $\psi_k^\star := \psi_k(\vstar_k)$. Moreover: \emph{(i)} $\Delta_k=+\infty$ iff rotor $k$ is essential (i.e., the submatrix $A_{-k}$ missing column $k$ has $\rank A_{-k}<m$); \emph{(ii)} $\sum_k\lambda_k=m$; \emph{(iii)} for symmetric designs (equal $\psi_i^\star$ and equal column leverages $A_{\bullet k}^\top(AA^\top)^{-1}A_{\bullet k}$, as on rotationally symmetric airframes) $\lambda_k=m/n$ and
\begin{equation}\label{eq:symgap}
    \Delta=\ln\!\left(\tfrac{n}{n-m}\right),
\end{equation}
which is a function of the redundancy ratio alone.
\end{theorem}
\begin{proof}
With $\psi_k=0$, Loewner monotonicity and the matrix determinant lemma applied to the rank-$1$ update of $S$ yield $\det D\le4^m\det S(1-\lambda_k)$, with equality when all other rotors are at $\vstar$. Since $\lambda_k$ is a diagonal entry of an orthogonal projector in the $S$-norm, $\lambda_k\in[0,1]$; $\lambda_k=1$ iff $\rank A_{-k}<m$, giving~(i). Part~(ii) follows because the projector has trace $m$; (iii) holds because equal $\psi^\star_i$ make $S=\psi^\star AA^\top$, so $\lambda_k$ is the column leverage, equal to $m/n$ by (ii).
\end{proof}

The leverage $\lambda_k$ measures how \emph{irreplaceable} rotor $k$ is, and crossing $v_k=0$ costs exactly $\Delta_k$; on a symmetric hexarotor every crossing costs $\ln3\approx1.10$\,nats regardless of size, propeller, or payload. The design goal follows: \emph{confine the state to one sign orthant and never pay this penalty.} The \emph{dropout level}, the highest readiness retained if any rotor approaches zero spin, is
\begin{equation}\label{eq:ldrop}
    \ldrop:=L^{\max}-\min_k\Delta_k.
\end{equation}

A floor $\lbar>\ldrop$ fences every zero-spin configuration, locking the vehicle inside its initial orthant. The gap depends on the capacities $\psi^\star$, not on $A$ alone: with $\pm30\%$ spread in $(\tbar_i,b_i)$ the hexarotor gaps scatter to $\{1.066,\,1.994,\,0.486\}$\,nats against the geometric $1.099$.

The floor's upper limit is the mission-dependent worst-case achievable readiness 
\begin{equation}\label{eq:hop}
    L^{\mathrm{op}}:=\min_{w\in\mathcal{W}}\,\max_{v\in f^{-1}(w)\cap\Ebar}L(v),
\end{equation}
so a certifiable floor must lie in the window
\begin{equation}\label{eq:window}
    \ldrop<\lbar\le L^{\mathrm{op}},
\end{equation}
non-empty iff $L^{\mathrm{op}}>\ldrop$; we set $\lbar=\ldrop+\kappa(L^{\mathrm{op}}-\ldrop)$, $\kappa\in(0,1)$.

\emph{Choosing $\kappa$.} The certified authority above dropout is $\lbar-\ldrop=\kappa(L^{\mathrm{op}}-\ldrop)$; the reserve $L^{\mathrm{op}}-\lbar=(1-\kappa)(L^{\mathrm{op}}-\ldrop)$ is the only margin for unmodeled effects: input delay (Corollary~\ref{cor:ceiling} needs $\eta(\tau_d)\le L^{\mathrm{op}}-\lbar$), sampled-data dips (Remark~\ref{rem:sampled}), and thrust-slope deformation (Corollary~\ref{cor:slope}). We take the largest $\kappa$ whose reserve covers them, $\kappa=1-(\eta(\tau_d)+\epsilon)/(L^{\mathrm{op}}-\ldrop)$ with $\epsilon$ the remaining budget, and default to $\kappa=1/2$ when they are unknown. Motor-parameter mismatch is \emph{not} a reason to lower $\kappa$: its metric shift (Corollary~\ref{cor:price}) usually exceeds the whole window and is handled by Theorem~\ref{thm:robust} (Section~\ref{sec:sim}).

\begin{proposition}[Certifiable Floor]\label{prop:achiev}
If \eqref{eq:window} is violated ($\lbar>L^{\mathrm{op}}$), some commanded fiber $f^{-1}(w)$ does not intersect the certified set, so no allocator can hold $h\ge0$ while tracking. Certifiability is thus a property of the task--vehicle pair, not the filter: a drone commanded near its thrust ceiling may have no certifiable floor.
\end{proposition}



\subsection{Readiness geometry and physical sweet spot}
\label{ssec:geom}

Each rotor enters $D$ only through $\psi_i=(v_i\abar_i)^2$, the product of \emph{thrust sensitivity} $v_i$ and \emph{torque headroom} $\abar_i$. Since $D$ is polynomial in $v_i^2$, $h$ inherits no kink from the modulus in $\phi$ and is $C^\infty$ wherever $D\succ0$.

\begin{lemma}[Readiness Gradient]\label{lem:grad}
Wherever $D(v)\succ0$, 
\begin{equation}\label{eq:grad}
    \begin{aligned}
    \frac{\partial h}{\partial v_i}(v) &=
    \frac{8\,v_i\,\abar_i(v_i)\bigl(\tbar_i-3b_iv_i^2\bigr)}{m_i}\,q_i(v),\\
    q_i &\coloneq A_{\bullet i}^\top D^{-1}A_{\bullet i}>0,
    \end{aligned}
\end{equation}
with $4\sum_i\psi_iq_i=m$. Each component vanishes exactly at $v_i\in\{0,\pm\vstar_i\}$ with $\vstar_i=v_i^{\mathrm{sat}}/\sqrt3$, and $L$ attains its global maximum $L^{\max}=\ln\det(4S)$ at the $2^n$ sweet-spot configurations $|v_i|=\vstar_i$.
\end{lemma}
\begin{proof}
Jacobi's formula applied to \eqref{eq:D} gives $\partial h/\partial v_i=4\psi_i'(v_i)q_i$; by standard differentiation $\psi_i'=2v_i\abar_i(\tbar_i-3b_iv_i^2)/m_i$, which is \eqref{eq:grad}. The trace identity $4\sum_i\psi_iq_i=\tr(D^{-1}D)=m$ is immediate. Roots of $\psi_i'$ are as stated. Global maximality follows from Loewner monotonicity of $D$ in each $\psi_i$ and monotonicity of $\det$ on the positive-definite cone; the $2^n$ degeneracy holds because $\psi_i$ is even in $v_i$.
\end{proof}

At the \emph{sweet spot} $\vstar_i\approx58\%\,v^\mathrm{sat}$ drag consumes one third of the torque budget. A rotor at $-\vstar$ is as ready as one at $+\vstar$, so a greedy selector has no tiebreaker and may jump, whereas the barrier holds authority inside one component.



\subsection{Certified safe set and forward invariance}
\label{ssec:inv}

A \emph{single} inequality $h(v)\ge0$ fences both failure modes, the zero-spin locus and the saturation boundary, with no auxiliary saturation logic (Figure~\ref{fig:nested}). We assume henceforth that $\lbar>\ldrop$ (\textbf{A2}, verified offline from \eqref{eq:gap}--\eqref{eq:ldrop}) and that $\lbar$ is a regular value of $h$ (\textbf{A3}, holds generically by Sard's theorem).

\begin{proposition}[Structure of the Certified-Authority Set]\label{prop:struct}
Under A1--A3, there exists $\varepsilon>0$ with $\psi_i(v_i)\ge\varepsilon$ for all $i$ on $\Cset=\{v\in\Ebar:h(v)\ge0\}$. Consequently: \emph{(i)} $\Cset\subset\operatorname{int}\Ebar$ and $|v_i|>0$ on $\Cset$; \emph{(ii)} $\rank J(v)=m$ and $D(v)\succ0$ on $\Cset$; \emph{(iii)} $\Cset$ is compact; \emph{(iv)} $\lambda_{\min}(D)\ge\underline\lambda:=e^{\lbar}/\bar\Lambda^{m-1}$ on $\Cset$, with $\bar\Lambda:=4\sum_i\psi^\star_i\|A_{\bullet i}\|^2$; \emph{(v)} $\partial\Cset$ is a smooth compact hypersurface with $\nabla h\ne0$.
\end{proposition}
\begin{proof}
For each $i$, Loewner monotonicity gives $\ln\det D\le g_i(\psi_i)$ where $g_i(0)=L^{\max}-\Delta_i\le\ldrop<\lbar$ by A2; by continuity $\psi_i\ge\varepsilon_i$ on $\Cset$. Set $\varepsilon=\min_i\varepsilon_i$: then $|v_i|>0$ and $\abar_i>0$, giving (i)--(ii). The set $K_\varepsilon:=\{v:\psi_i\ge\varepsilon\ \forall i\}$ is closed and bounded in $\overline\Ebar$, hence compact; $\Cset\subseteq K_\varepsilon$ gives (iii). For (iv): $\det D\ge e^{\lbar}$ and $\lambda_k(D)\le\bar\Lambda$. Item (v) is A3 with the regular value theorem.
\end{proof}

Any trajectory in $\Cset$ thus respects the actuator limits without an explicit saturation check; Figure~\ref{fig:nested} shows the four components for $n=2$.

\begin{figure}[tb]\centering
    \centering
    \resizebox{0.64\columnwidth}{!}{%
        \input{figures/tikz/fig_nested.tex}}
    \vspace{-0.7em} 
    \caption{Certified set $\Cset$ for $n=2$ ($\tbar_i=b_i=m_i=1$): the four sign-orthant components (blue) sit strictly between the zero-spin hyperplanes $v_i=0$ (orange) and saturation boundary $\partial\Ebar$ (red dashed), arrows mark the uniform margin on both sides (Proposition~\ref{prop:struct}(i),(iv)).}
    \label{fig:nested}
\end{figure}

It remains to show that a deliverable torque always keeps the state in $\Cset$. Drag makes this delicate: braking capacity $(\tbar_i+b_iv_i^2)/m_i$ grows with spin while accelerating capacity $\bar{a}_i$ vanishes at saturation (nearly $50\!:\!1$ at $0.98\,v^{\mathrm{sat}}$), so authority must be measured in the binding symmetric direction.

\begin{lemma}[Viability via Symmetric Acceleration Capacity]\label{lem:sac}
For every $v\in\operatorname{int}\Ebar$, the interval $[-\abar_i,\abar_i]$ is the largest symmetric box of admissible accelerations within $[a_i^-,a_i^+]$. Consequently,
\begin{equation}\label{eq:viab}
    \sup_{|\tau|\le\tbar}\nabla h(v)^\top\bigl(d(v)+M^{-1}\tau\bigr)
    \ge\sum_i|\partial_ih|\,\abar_i\ge0,
\end{equation}
with strict inequality whenever $\nabla h(v)\ne0$.
\end{lemma}
\begin{proof}
Evaluating $a_i^\pm=(\pm\tbar_i-b_iv_i|v_i|)/m_i$ for all signs of $v_i$ confirms $[-\abar_i,\abar_i]\subseteq[a_i^-,a_i^+]$, tight at one endpoint. Choosing $a_i=\abar_i\operatorname{sgn}(\partial_ih)$ gives \eqref{eq:viab}.
\end{proof}

Lemma~\ref{lem:sac} makes the \ac{CBF} admissible set non-empty; it hinges on weighting $D$ by the symmetric $\abar_i$ rather than the braking margin.



\subsection{Control barrier allocation filter}
\label{ssec:filter}

The filter is a \emph{post-processor}: it returns the torque closest to a nominal command $\tau_{\mathrm{ref}}$ (e.g., effort-minimizing) that keeps the state in $\Cset$, and reproduces $\tau_\mathrm{ref}$ exactly when there is no conflict. Let $\mu:=\dot w_{\mathrm{des}}+K_w\tilde w$ be the task-acceleration demand, with wrench error $\tilde w:=w_{\mathrm{des}}-A\phi(v)$ and gain $K_w = k_w I$, $k_w > 0$. The allocation filter solves:
\begin{equation}\label{eq:qp}
    \begin{aligned}
    \min_{\tau,\delta}\;&\|\tau-\tau_{\mathrm{ref}}\|_R^2+\rho\|\delta\|^2\\
    \text{s.t.}\;&\nabla h(v)^\top a\ge-\alpha(h(v)),\\
    &J(v)\,a=\mu+\delta,\\
    &|\tau|\le\tbar,\\
    &a=d(v)+M^{-1}\tau.
    \end{aligned}
\end{equation}
Here $R\succ0$ is diagonal and the slack $\delta$, penalized by $\rho$, absorbs the wrench deficit when barrier and task conflict; through $\abar_i$ in Lemma~\ref{lem:grad}, $\nabla h$ weights each rotor by its remaining capacity.

\begin{lemma}[Uniform Strict Feasibility]\label{lem:slater}
Under A1--A3, $\exists\sigma>0:\;\max_{|\tau|\le\tbar}[\nabla h^\top(d+M^{-1}\tau)+\alpha(h)]\ge\sigma$ uniformly on $\Cset$. The minimizer of \eqref{eq:qp} is unique and locally Lipschitz on $\Cset$.
\end{lemma}
\begin{proof}
If $\nabla h\ne0$, Lemma~\ref{lem:sac} supplies a strictly interior torque; if $\nabla h=0$ then $v\in\operatorname{int}\Cset$ (Proposition~\ref{prop:struct}(v)) and the constraint holds trivially. Uniformity follows from compactness of $\Cset$. The minimizer is Lipschitz by strong convexity and Hoffman's bound under the uniform Slater condition \cite{Robinson1980strong}.
\end{proof}

\begin{theorem}[Valid \ac{CBF} and Forward Invariance]\label{thm:inv}
Under A1--A3, $h$ is a valid \ac{CBF} on $\Cset$ and \eqref{eq:qp} is feasible for all $\rho>0$. For any locally Lipschitz selection $\tau(t,v)$ satisfying the barrier row with $v(0)\in\Cset$, the closed-loop trajectory is unique, complete, and satisfies $v(t)\in\Cset\subset\operatorname{int}\Ebar$ for all $t\ge0$, remaining in the connected sign-orthant component containing $v(0)$.
\end{theorem}
\begin{proof}
The admissible set is non-empty by Lemma~\ref{lem:sac} (when $\nabla h\ne0$) or trivially (when $\nabla h=0\in\operatorname{int}\Cset$). Feasibility of \eqref{eq:qp} holds because $\delta$ is unconstrained. The minimizer is locally Lipschitz (Lemma~\ref{lem:slater}), guaranteeing local existence and uniqueness. Along any solution $\dot h\ge-\alpha(h)$; applying the standard comparison lemma yields $h(v(t))\ge0$ \cite{Ames2017CBFQP, Blanchini1999setinvariance}. Compactness precludes finite escape, and interior containment follows from Proposition~\ref{prop:struct}(i).
\end{proof}

The barrier row acts dynamically as a \emph{margin-dependent speed limit} on authority expenditure, inactive deep inside $\Cset$ and tightening to zero at $\partial\Cset$. The filter never evaluates the set-valued $\arg\max_{f^{-1}(w)}L$: a strongly convex \ac{QP} yields a single continuous selection.

\begin{remark}[Sampled-data implementation]\label{rem:sampled}
Theorem~\ref{thm:inv} certifies invariance in continuous time. Under zero-order hold with step $\Delta t$, the barrier row holds only at samples, so between them $h$ may dip by at most $\alpha(h)\Delta t+\tfrac12 L_{\dot h}\Delta t^2$---the first-order term being the permitted rate $-\dot h(t_k)\le\alpha(h(t_k))$ and $L_{\dot h}$ bounding $\ddot h$ on $\Cset$. With the row active, the dip settles at $\mathcal{O}(L_{\dot h}\Delta t/\alpha)$, which the strict interior absorbs exactly as it absorbs transport delay. At $\Delta t=10^{-3}T_0$ we observe at most $5\times10^{-3}$\,nats (Section~\ref{ssec:ff}), and the dip grows linearly with $\Delta t$. A discrete-time \ac{CBF} \cite{Agrawal2017discrete} removes the $\Delta t$ dependence at the cost of a nonconvex per-step check.
\end{remark}



\subsection{Task relaxation and robustness}
\label{ssec:tracking}

\subsubsection{Tracking bounds and the zero-sum exchange}
Whether the filter can simultaneously maintain the floor \emph{and} track the task exactly depends on the available null-space geometry. Define the maximum authority ascent achievable at constant wrench,
\begin{equation}\label{eq:gamma}
    \gamma(v):=\max\bigl\{\nabla h^\top\xi:\xi\in\Ker J(v),\;|\xi_i|\le\abar_i\bigr\}\ge0,
\end{equation}
and the \emph{alignment ratio}
\begin{equation}\label{eq:theta}
    \vartheta(v):=\gamma(v)/\|W^{1/2}\nabla h(v)\|_2\in[0,\sqrt{n}].
\end{equation}

On a symmetric platform at near-uniform spin, $\nabla h\propto\mathbf{1}=A^\top e_{\mathrm{coll}}\in\Range J^\top$, so $\gamma=0$: any redistribution that speeds some rotors slows others, and the gains cancel. $\vartheta$ measures the distance from this zero-sum regime.

\begin{theorem}[Strict Feasibility and Bounded Task Relaxation]%
\label{thm:feas}\label{thm:cont}
Let $\underline\lambda$ be as in Proposition~\ref{prop:struct}(iv) and suppose $\|\mu\|\le(1-\beta)\sqrt{\underline\lambda}$ for some capacity split ratio $\beta\in(0,1)$. If
\begin{equation}\label{eq:feascond}
    \beta\gamma(v)+\alpha(h(v))\ge(1-\beta)\|W^{1/2}\nabla h(v)\|_2,
\end{equation}
then \eqref{eq:qp} has a feasible point with $\delta=0$. On $\partial\Cset$ this reduces to $\vartheta(v)\ge(1-\beta)/\beta$. When $\delta^\star\ne0$ the minimizer satisfies
\begin{equation}\label{eq:deltabound}
    \|\delta^\star\|\le c_\tau/\sqrt\rho,\qquad c_\tau:=2\|R\|^{1/2}\|\tbar\|,
\end{equation}
and the wrench error obeys $\dot{\tilde w}=-K_w\tilde w-\delta^\star$, yielding
\begin{equation}\label{eq:wtrack}
    \|\tilde w(t)\|\le e^{-k_wt}\|\tilde w(0)\|+c_\tau/(k_w\sqrt\rho);
\end{equation}
the ultimate wrench error is $\mathcal{O}(\rho^{-1/2})$.
\end{theorem}
\begin{proof}
Strict feasibility: set $a=a_0+\xi$ with $a_0=WJ^\top D^{-1}\mu$ the horizontal lift and $\xi$ feasible for the program \eqref{eq:gamma} scaled by $\beta$. Then $Ja=\mu$ and $\delta=0$. Admissibility of $a$ follows from $\|W^{-1/2}a_0\|\le(1-\beta)$ (Proposition~\ref{prop:struct}(iv)) and $|\xi_i|\le\beta\abar_i$; the barrier row follows from \eqref{eq:feascond} via Cauchy--Schwarz. When $\delta^\star\ne0$: comparing the optimal value with the 
feasible point $(\tau^{(0)},0)$, $\tau^{(0)}$ the torque realizing $a$ above, from the strict-feasibility case and using strong convexity of the \ac{QP} objective gives \eqref{eq:deltabound}. The error dynamics $\dot{\tilde w}=-K_w\tilde w-\delta^\star$ yield \eqref{eq:wtrack} by the standard ISS bound.
\end{proof}

When $\vartheta\to0$ the margin comes from the strict interior $\alpha(h)>0$, guaranteed by $\lbar<L^{\mathrm{op}}$ ($\kappa<1$). The bound is a design rule (quadruple $\rho$ to halve the error), and $\delta$ enters the outer loops as a small bounded disturbance instead of an impulsive one.

\subsubsection{Parametric robustness}\label{sec:robust}
Motor parameters vary with temperature, battery state, and wear, which affects both the \emph{metric}, as the true $L(v;\zeta)$ may lie below the model, and the \emph{drift}: the barrier is enforced with the nominal $d(v;\hat\zeta)$, $\hat\zeta=(\tbar,b)$, while the plant runs $d(v;\zeta)$, $\zeta\in\mathcal{Z}_p$. Bounding a metric does not control its derivative, so both are treated.

Since $\abar_i$ is affine in $\tbar_i$ and $b_i$, a weak motor and a draggy propeller compound; with relative half-width $p\in(0,1)$ on both, the worst case is the single corner $\zeta^-=(\tbar(1-p),b(1+p))$ instead of $2^{2n}$ vertices.

\begin{theorem}[Worst-Case Metric and Robust Forward Invariance]\label{thm:robust}\label{thm:robinv}
Let $\mathcal{Z}_p=\{(\tbar',b'):\tbar'_i\in[\tbar_i(1\pm p)], b'_i\in[b_i(1\pm p)]\}$,
$\abar^-_i:=(\tbar_i(1-p)-b_i(1+p)v_i^2)/m_i$, $D^-:=4A\diag(v_i^2(\abar^-_i)^2)A^\top$, and
$\Ebar^-:=\{v:|v_i|<\sqrt{\tbar_i(1-p)/(b_i(1+p))}\}\subsetneq\Ebar$. For all $\zeta\in\mathcal{Z}_p$ and $v\in\Ebar^-$: $D(v;\zeta)\succeq D^-(v)$, so $L(v;\zeta)\ge L^-(v):=\ln\det D^-(v)$. Moreover, if $\tau$ satisfies the tightened barrier row
\begin{equation}\label{eq:robustcbf}
    \nabla h_p^\top M^{-1}\tau\ge-\alpha(h_p)
    -\min_{\zeta\in\mathcal{Z}_p}\nabla h_p^\top d(v;\zeta),
\end{equation}
where $h_p:=L^--\lbar$, the row being enforced with the worst-case available torque $|\tau_i|\le\tbar_i(1-p)$, and
\begin{equation}\label{eq:driftmin}
    \min_{\zeta\in\mathcal{Z}_p}\nabla h_p^\top d(v;\zeta)
    =\nabla h_p^\top d(v;\hat\zeta)-p\|c\|_1,
\end{equation}
where $c_i:=-\partial_ih_p\cdot b_iv_i|v_i|/m_i$. Hence, $\Cset_p:=\{v\in\Ebar^-:h_p(v)\ge0\}$ is forward invariant for every $\zeta\in\mathcal{Z}_p$.
\end{theorem}
\begin{proof}
On $\Ebar^-$, $\abar_i(\zeta)\ge\abar^-_i\ge0$; squaring (both non-negative) and applying Loewner monotonicity and log-det monotonicity gives the metric bound. For invariance: $\dot h_p=\nabla h_p^\top(d(\zeta)+M^{-1}\tau)\ge\nabla h_p^\top M^{-1}\tau+\min_{\zeta'}\nabla h_p^\top d(\zeta')\ge-\alpha(h_p)$ by \eqref{eq:robustcbf}; the comparison lemma gives $h_p\ge0$. Feasibility of \eqref{eq:robustcbf} under $|\tau_i|\le\tbar_i(1-p)$ follows from Lemma~\ref{lem:sac} applied to $\abar^-_i$, so the reduced box still admits an authority-preserving input. The minimum in \eqref{eq:driftmin} is attained at a sign vertex because the drift component $d_i=-b_iv_i|v_i|/m_i$ is affine in $b_i$.
\end{proof}

\begin{corollary}[Exact Price of Robustness]\label{cor:price}
At the sweet-spot maximizer, $\psi^\star_i=4\tbar_i^3/(27b_im_i^2)$ scales by $(1-p)^3/(1+p)$ at $\zeta^-$, so the floor shift is independent of $A$:
\begin{equation}\label{eq:price}
    L^{\max}_p-L^{\max}=m\ln\!\left(\frac{(1-p)^3}{1+p}\right).
\end{equation}
The certified authority \emph{volume} $\det D^{1/2}$ scales by $\bigl((1-p)^3/(1+p)\bigr)^{m/2}$ and the per-direction radius by its $m$-th root. On an $m=4$ hexarotor the retained volume is $44\%$, $18\%$, $7\%$ at $p=0.1,0.2,0.3$ (radius $81\%$, $65\%$, $51\%$).
\end{corollary}

The drift term \eqref{eq:driftmin} costs one $\mathcal{O}(n)$ pass, and the price \eqref{eq:price} is attained at $\zeta^-\in\mathcal{Z}_p$, so no method certifying the full box can recover more. Torque enters cubed, so torque calibration is worth about $3\times$ drag calibration.

Away from hover, inflow also changes the \emph{thrust} map, which the box $\mathcal{Z}_p$ does not cover. Its effect separates cleanly from the motor-level certificate.

\begin{corollary}[Thrust-Slope Deformation]\label{cor:slope}
Let the true thrust $\tilde\phi_i$ of rotor $i$ (normalized as $\phi_i$) have slope $\partial\tilde\phi_i/\partial v_i=2|v_i|\,s_i$, where $s_i\in[\underline s,\bar s]$, $\underline s>0$, collects aerodynamic effects absent from $\phi$ (e.g., advance ratio in forward flight). Then the true co-metric $D_s=4A\diag(s_i^2\psi_i)A^\top\succeq\underline s^{2}D$, so $L_s(v)\ge L(v)+2m\ln\underline s$, while Theorem~\ref{thm:inv} is unaffected because $h$ depends on motor states only.
\end{corollary}

Loewner monotonicity gives the bound directly. Non-hover flight therefore acts through two channels: the \emph{mission}, whose larger wrench demand lowers $L^{\mathrm{op}}$ (Proposition~\ref{prop:achiev}), and the \emph{slope}, which costs at most $-2m\ln\underline s$ of the reserve or is removed by scheduling the barrier on the measured airspeed, $h_s:=\ln\det D_s-\lbar$, with $\partial_t h_s$ added to the barrier row.

\subsubsection{Unmodeled input delay}\label{sec:delay}
Under transport delay $\tau_d$ the \ac{QP} acts on a state the vehicle has already left: the drift and the gradient are stale, and $h$ can keep falling for a whole delay window.

\begin{proposition}[Delay Undershoot]\label{prop:delay}
Let the applied torque be $\tau^\star(v(t-\tau_d))$ and define $K_1:=\sup_{v \in \Cset} \|\nabla(\nabla h^\top d)(v)\|$, $K_2:=\sup\|\nabla^2h\|_\infty$, $V:=\max_i\sup(|d_i|+\tbar_i/m_i)$, $\bar H:=L^{\max}-\lbar$, and
$K:=K_1+K_2\max_i(\tbar_i/m_i)$. For a linear class-$\mathcal{K}$ function $\alpha(h)=\alpha h$,
\begin{equation}\label{eq:eta}
    h(v(t))\ge-\eta(\tau_d),
\end{equation}
where $\eta(\tau_d):=(\alpha\bar H+KV\tau_d)\tau_d+\tfrac{KV\tau_d}{\alpha}$ is the worst-case undershoot bound.
\end{proposition}
\begin{proof}
Decompose $\dot h(t)$ around the delayed state $v^-$. The stale barrier row contributes $\ge-\alpha h(v^-)$; the two mismatch brackets are bounded by $K_1V\tau_d$ and $K_2V\tau_d\max_i(\tbar_i/m_i)$. Using $h\le\bar H$ and the standard comparison lemma yields \eqref{eq:eta}.
\end{proof}

\begin{corollary}[Certified Delay Ceiling]\label{cor:ceiling}
Raising the floor to $\lbar+\eta(\tau_d)$ certifies the original floor $\lbar$ under delay $\tau_d$ iff $\eta(\tau_d)\le L^{\mathrm{op}}-\lbar$. The certified ceiling $\tau_d^\star$ is the unique root of $\eta(\tau_d^\star)=L^{\mathrm{op}}-\lbar$. Beyond $\tau_d^\star$, the raised floor exceeds $L^{\mathrm{op}}$ and, by Proposition~\ref{prop:achiev}, cannot be held while tracking.
\end{corollary}

After a rotor failure, $\Delta_k$ is recomputed on $A_{-k}$ (newly essential rotors get $\Delta_j=+\infty$) and an empty window is never enforced. Algorithm~\ref{alg:filter} summarizes the method.

\begin{algorithm}[t]
    \SetAlgoLined\DontPrintSemicolon\small
    \KwIn{$A,(\tbar,b,m),\lbar,\alpha,K_w,\rho,R$}
    \textbf{Offline:} $\lambda_k\leftarrow\psi^\star_kA_{\bullet k}^\top S^{-1}A_{\bullet k}$;
    $\quad\ldrop\leftarrow L^{\max}-\min_k(-\ln(1-\lambda_k))$; verify~\eqref{eq:window}\;
    \While{running}{
      measure $v$; compute $D,J,W,\nabla h$~\eqref{eq:grad}, $\vartheta$~\eqref{eq:theta}\;
      $\tilde w\leftarrow w_{\mathrm{des}}-A\phi(v)$;
      $\quad\mu\leftarrow\dot w_{\mathrm{des}}+K_w\tilde w$;
      $\quad\tau_{\mathrm{ref}}\leftarrow$ nominal allocator (e.g., effort-min \ac{QP})\;
      solve~\eqref{eq:qp} for $(\tau^\star,\delta^\star)$; apply $\tau^\star$\;
    }
    \caption{Readiness-barrier allocation filter}\label{alg:filter}
\end{algorithm}



\section{Simulation Study}
\label{sec:sim}

We simulate a coplanar hexarotor ($n{=}6$, $m{=}4$, $\Delta=\ln3$) and a fully-actuated tilted octorotor ($n{=}8$, $m{=}6$, $\Delta=\ln4$), with identical rotors, so $\lambda_k=m/n$ and Theorem~\ref{thm:gap} applies exactly. A fully-actuated hexarotor ($n{=}m{=}6$) is excluded: every rotor is essential ($\lambda_k=1$, $\Delta_k=+\infty$) and no stratum jump can occur. The study is nondimensional: $\tbar_i=b_i=m_i=1$, so speeds are in units of $v^{\mathrm{sat}}$ and time in units of the motor time scale $T_0:=m_i/\sqrt{\tbar_ib_i}$, the time full torque needs to reach $v^{\mathrm{sat}}$ against no drag. For the measured propeller--motor pair of Section~\ref{ssec:hw}, $T_0=0.32$\,s. Equation~\eqref{eq:motor} is integrated at $\Delta t=10^{-3}T_0$, \eqref{eq:qp} is solved with OSQP\footnote{\url{https://github.com/osqp}} ($\alpha(h)=5h$, $k_w=8$, $\rho=10^4$, $R=I$), and the code reproducing every table and figure is available at \texttt{\coderepo}.

Numerically, the hexarotor gives $L^{\max}=-10.128$, $\ldrop=-11.226$ ($\Delta=1.098\approx\ln3$) and the octorotor $L^{\max}=+1.448$, $\ldrop=+0.062$ ($\Delta=1.386\approx\ln4$), both matching \eqref{eq:symgap} to machine precision; the window $(\ldrop,L^{\max})$ exceeds one nat on both.



\subsubsection{Singularity avoidance and closed-loop tracking}
\label{ssec:singexp}

Four allocators drive the hexarotor through a sinusoidal roll-moment reversal (period $2T_0$) at fixed collective: the effort-minimizing \ac{QP} \cite{Johansen2013allocation}, the greedy readiness maximizer of \cite{Franchi2026aeropromptness} (fiberwise $\arg\max h$, global multi-start re-solved every $0.02\,T_0$), its low-pass filtered variant (first order, time constant $0.05\,T_0$), and the proposed filter (Algorithm~\ref{alg:filter}, $\kappa=1/2$); see Table~\ref{tab:sweep} and Figures~\ref{fig:thrust}--\ref{fig:lpf}. A moment reversal, the elementary event of any attitude maneuver, makes the optima $\pm\vstar$ exchange dominance.

\begin{table}[tb]\centering\small
    \caption{Collective-thrust sweep (hexarotor, moment amplitude $1.5\times$ ref, $\lbar$ per Proposition~\ref{prop:achiev}). ``thr.'': collective as a fraction of $n\phi(v^\star_i)$, all rotors at the sweet spot; moment amplitudes in units of $\tfrac14\phi(v^\star_i)$; ``TV'': total variation of the commanded allocation; ``wErr'': \acs{RMS} wrench error; $h_{\min}=\min_t h$. Bold: floor violated.}
    \label{tab:sweep}
    \vspace{-0.5em}
    \renewcommand{\arraystretch}{1.15}
    \setlength{\tabcolsep}{3pt}
    \begin{adjustbox}{max width=0.98\columnwidth}
    \begin{tabular}{cc|ccc|ccc|ccc|ccc}
    \toprule
     & & \multicolumn{3}{c|}{effort-min} & \multicolumn{3}{c|}{greedy DAAM} & \multicolumn{3}{c|}{low-pass DAAM} & \multicolumn{3}{c}{Algorithm~\ref{alg:filter}}\\
    thr. & $\lbar$ & $h_{\min}$ & TV & wErr & $h_{\min}$ & TV & wErr & $h_{\min}$ & TV & wErr & $h_{\min}$ & TV & wErr\\
    \midrule
    0.6 & \multicolumn{13}{c}{not certifiable: $L^{\mathrm{op}}=-11.49<\ldrop=-11.23$}\\
    0.7 & $-11.14$ & $+0.08$ & 2.8 & 0.0003 & $\mathbf{-0.54}$ & 31.5 & 0.1449 & $\mathbf{-1.23}$ & 13.8 & 0.1785 & $+0.12$ & 2.7 & 0.0018\\
    0.8 & $-11.00$ & $+0.22$ & 2.6 & 0.0003 & $\mathbf{-0.48}$ & 8.4 & 0.1132 & $\mathbf{-0.43}$ & 7.6 & 0.1149 & $+0.22$ & 2.6 & 0.0003\\
    0.9 & $-10.91$ & $+0.31$ & 2.5 & 0.0003 & $+0.31$ & 2.4 & 0.0069 & $+0.32$ & 2.3 & 0.0198 & $+0.31$ & 2.5 & 0.0003\\
    1.0 & $-10.87$ & $+0.35$ & 2.3 & 0.0003 & $+0.35$ & 2.3 & 0.0069 & $+0.36$ & 2.2 & 0.0198 & $+0.35$ & 2.3 & 0.0003\\
    1.1 & $-10.87$ & $+0.35$ & 2.2 & 0.0003 & $+0.36$ & 2.2 & 0.0069 & $+0.36$ & 2.1 & 0.0198 & $+0.35$ & 2.2 & 0.0003\\
    1.2 & $-10.90$ & $+0.33$ & 2.1 & 0.0003 & $+0.33$ & 2.1 & 0.0069 & $+0.34$ & 2.0 & 0.0198 & $+0.33$ & 2.1 & 0.0003\\
    \bottomrule
    \end{tabular}
    \end{adjustbox}
\end{table}

\begin{figure}[tb]
    \centering
    \resizebox{\columnwidth}{!}{%
        \input{figures/tikz/fig_sweep.tex}}
    \vspace*{-1.7em}
    \caption{Collective-thrust sweep. Left: worst-case readiness margin $\min_t h$ (solid) and its time average $\overline h$ (dashed); right: \ac{RMS} wrench error.}
    \label{fig:thrust}
\end{figure}

\begin{figure}[tb]
    \centering
    \resizebox{\columnwidth}{!}{%
        \input{figures/tikz/fig_lpf2.tex}}
    \vspace*{-1.9em}
    \caption{Moment reversal at collective $0.75$. Left: peak commanded actuator rate (maximum over each $0.02\,T_0$ greedy update period, so no jump is lost to decimation); right: readiness margin $h(t)$. Effort-min and proposed overlap.}
    \label{fig:lpf}
\end{figure}

At collective $0.7$--$0.8$ greedy DAAM violates the floor it maximizes ($h_{\min}=-0.54$, $-0.48$): each handover between $\pm\vstar$ is an orthant crossing through the zero-spin dead zone, TV reaches $31.5$ against $2.7$, and the motor lag turns the jump into a wrench error of $0.145$ against $0.0018$ (about $79\times$). Where authority is abundant ($\ge0.9$) the barrier is inactive and the filter reproduces its nominal allocator exactly; where active it raises worst-case readiness over effort-min ($+0.12$ vs.\ $+0.08$).

Low-pass filtering cuts the peak commanded rate from $731$ to $14.6\,v^{\mathrm{sat}}/T_0$ (Figure~\ref{fig:lpf}), but the filtered reversal travels the same path, only more slowly: time below the floor grows from $0.63$ to $0.76\,T_0$, and at collective $0.7$ the violation deepens to $-1.23$. The proposed filter peaks at $0.52$, keeps $\min_{i,t}|v_i|=0.335\,v^{\mathrm{sat}}$, and along it $\vartheta\le7\times10^{-3}$: the margin comes from $\alpha(h)>0$, not null-space ascent.

\emph{Does greedy maximization buy anything?} Where it does not jump it should lead on \emph{average} readiness $\overline h$ (Figure~\ref{fig:thrust}, dashed). For collective $\ge0.9$ greedy and low-pass lead by at most $0.001$ and $0.006$\,nats, at $25$--$70\times$ the wrench error; in the scarce band they trail ($\overline h=+0.14$, $+0.00$ vs.\ $+0.41$). The lead is small since a symmetric fiber allows no readiness ascent ($\vartheta\approx0$); on heterogeneous platforms a readiness-ascent $\tau_{\mathrm{ref}}$ inside the filter could capture it without jumps (future work).



\subsubsection{Monte-Carlo mission sweeps}
\label{ssec:monteCarlo}

Random certifiable missions (uniform collective and amplitude), $100$ per hexarotor set and $50$ per octorotor set, are flown by all four allocators (Table~\ref{tab:mc}). Over the hexarotor envelope the barrier is active $5\%$ of the time, yet greedy and low-pass DAAM already violate the floor in $13\%$ and $19\%$ of missions (effort-min once, by $-0.002$). In the scarce corner they violate in $51\%$ and $62\%$ of missions, as deep as $-5.7$ and $-9.4$, with mean wrench error $75$ and $86\times10^{-3}$; the proposed filter never violates (lowest $h_{\min}=+0.11$). The octorotor rarely forces a handover (one scarce mission in $50$), but that one collapses to $h_{\min}=-12$ and $-23$.

\begin{table}[tb]\centering\small
    \caption{Monte-Carlo missions (mean over missions, $\pm$ 95\% CI half-width). ``viol.'': fraction of missions with $\min_t h<0$; wErr in units of $10^{-3}$.}
    \label{tab:mc}
    \vspace{-0.5em}
    \renewcommand{\arraystretch}{1.05}
    \setlength{\tabcolsep}{2.5pt}
    \begin{adjustbox}{max width=0.98\columnwidth}
    \begin{tabular}{l|ccc|ccc}
    \toprule
    & \multicolumn{3}{c|}{envelope} & \multicolumn{3}{c}{authority-scarce corner}\\
    allocator & viol. & $h_{\min}$ & wErr & viol. & $h_{\min}$ & wErr\\
    \midrule
    \multicolumn{7}{l}{\emph{hexarotor} ($100$ + $100$ missions; 2 + 13 rejected)}\\
    effort-min & $\mathbf{0.01}$ & $+0.29{\pm}0.02$ & $0.3{\pm}0.0$ & $0.00$ & $+0.19{\pm}0.02$ & $2.5{\pm}3.1$\\
    greedy DAAM & $\mathbf{0.13}$ & $+0.22{\pm}0.05$ & $23.7{\pm}7.6$ & $\mathbf{0.51}$ & $-0.29{\pm}0.20$ & $74.7{\pm}14.4$\\
    low-pass DAAM & $\mathbf{0.19}$ & $+0.19{\pm}0.07$ & $36.9{\pm}7.3$ & $\mathbf{0.62}$ & $-0.51{\pm}0.30$ & $85.9{\pm}13.7$\\
    Algorithm~\ref{alg:filter} & $0.00$ & $+0.30{\pm}0.02$ & $0.6{\pm}0.2$ & $0.00$ & $+0.21{\pm}0.01$ & $3.4{\pm}3.1$\\
    \midrule
    \multicolumn{7}{l}{\emph{octorotor} ($50$ + $50$ missions; 0 + 15 rejected)}\\
    effort-min & $0.00$ & $+0.60{\pm}0.02$ & $0.3{\pm}0.0$ & $0.00$ & $+0.23{\pm}0.04$ & $0.4{\pm}0.0$\\
    greedy DAAM & $0.00$ & $+0.60{\pm}0.02$ & $6.9{\pm}0.4$ & $\mathbf{0.02}$ & $-0.01{\pm}0.48$ & $12.6{\pm}6.2$\\
    low-pass DAAM & $0.00$ & $+0.60{\pm}0.02$ & $20.1{\pm}1.0$ & $\mathbf{0.02}$ & $-0.23{\pm}0.92$ & $30.9{\pm}7.1$\\
    Algorithm~\ref{alg:filter} & $0.00$ & $+0.60{\pm}0.02$ & $0.3{\pm}0.0$ & $0.00$ & $+0.23{\pm}0.04$ & $0.4{\pm}0.0$\\
    \bottomrule
    \end{tabular}
    \end{adjustbox}
\end{table}



\subsubsection{Choosing the floor}

Raising $\kappa$ from $0.1$ to $0.9$ multiplies the certified authority ninefold ($0.03\to0.30$\,nats) but cuts $\tau_d^\star$ from $3.5$ to $0.4\times10^{-3}T_0$; activity rises from $13\%$ to $28\%$, wrench error stays below $4\times10^{-3}$, and a $0.1\,T_0$ delay is survived only for $\kappa\le0.4$. Under $10\%$ motor mismatch the nominal filter fails on $25$--$35\%$ of plants for every $\kappa$: mismatch belongs to Theorem~\ref{thm:robust}.



\subsubsection{Where the certificate applies: mission demand and forward flight}
\label{ssec:ff}

Figure~\ref{fig:applic} (left) maps the certifiability window of Proposition~\ref{prop:achiev}, which closes at low collective and high moment demand; above an amplitude of about $2.5\times$, no collective admits a floor. Inside the frontier every commanded fiber meets $\Cset$, so the barrier excludes no wrench of the mission; outside, no allocator can hold a floor above $\ldrop$ while tracking, and the filter trades tracking for authority through $\delta$.

\emph{Forward flight.} Each vehicle accelerates from hover to airspeed $V_\infty\le15$\,m/s over $10\,T_0$. The plant uses a first-order inflow model $C_T(J)=C_{T0}(1-J/J_0)$ in the advance ratio, with $J_0=0.65$ and $v^{\mathrm{sat}}=7330$\,rpm taken from the measured propeller data of Section~\ref{ssec:hw}, so $s_i=1-J_i/(2J_0)$ in Corollary~\ref{cor:slope}. The command adds a drag force $0.25\,W(V_\infty/15)^2$, an integral loop on the true wrench stands in for the position controller, and the floor is the hover design. The motor-level barrier stays invariant at every speed, but on the hexarotor the hover model is optimistic: the true margin $h_s$ falls below the floor from $12$\,m/s for effort-min, greedy DAAM, and the hover-model filter alike (Figure~\ref{fig:applic}, right). Scheduling the barrier on the measured airspeed keeps $h_s\ge-2\times10^{-3}$ (the sampled-data dip) up to $12$\,m/s at a wrench error below $3\times10^{-2}$, beyond the speed ($\approx9.3$\,m/s) where the mission stops being certifiable in the true metric. On the octorotor, certifiability ends near $8.4$\,m/s; beyond it greedy DAAM drives a rotor through zero spin ($h_s=-5.0$ at $12$\,m/s) and effort-min collapses at $15$\,m/s ($h_s=-7.8$), while both filter variants keep $h_s\ge-0.003$ (Remark~\ref{rem:sampled}) by giving up tracking (\ac{RMS} up to $0.44$).

\begin{figure}[tb]
    \centering
    \resizebox{\columnwidth}{!}{%
        \input{figures/tikz/fig_applic.tex}}
    \vspace*{-1.7em}
    \caption{Left: certifiability window $L^{\mathrm{op}}-\ldrop$ of the hexarotor over collective and moment amplitude (white line: frontier, Proposition~\ref{prop:achiev}; boxes: Monte-Carlo sampling regions). Right: hexarotor accelerated to airspeed $V_\infty$ with the hover-designed floor; true readiness margin $\min_t h_s$ (Corollary~\ref{cor:slope}); effort-min and the hover-model filter overlap; shaded: mission not certifiable in the true metric.}
    \label{fig:applic}
\end{figure}



\subsubsection{Resilience to unmodeled delay}

Figure~\ref{fig:delay} sweeps torque-path delay to failure. Below $\tau_d^\star=1.97\times10^{-3}\,T_0$ (Corollary~\ref{cor:ceiling}) the margined floor $\lbar+\eta(\tau_d)$ protects the state; beyond it the raised floor would exceed $L^{\mathrm{op}}$. The plain floor degrades mildly ($\min_t h=+0.140$ at $0.02\,T_0$, $+0.098$ at $0.05\,T_0$) and fails near $0.09\,T_0$; the $25$--$45\times$ gap reflects the worst-case constant $KV\approx400\,T_0^{-1}$. With $T_0=0.32$\,s, typical \ac{ESC} delays ($1$--$10$\,ms, i.e.\ $0.003$--$0.03\,T_0$) are tolerated, not certified; certifying them needs a trajectory-wise bound.

\begin{figure}[tb]
    \centering
    \resizebox{.72\columnwidth}{!}{%
        \input{figures/tikz/fig_delayregimes.tex}}
    \vspace*{-0.75em}
    \caption{Delay sweep (hexarotor, collective $0.75$). Shaded: certified ($\tau_d\le\tau_d^\star$), tolerated, and failure regimes. The margined floor is enforceable only for $\tau_d\le\tau_d^\star$; the dotted line is the delay-free value.}
    \label{fig:delay}
\end{figure}



\subsubsection{Verification of worst-case robust invariance}

We subject $(\tbar_i,b_i)$ to $\pm p$ mismatch on $50$ random plants per level, shared by all controllers (Table~\ref{tab:robust}). The nominal certificate fails on $46\%$ at $p=10\%$ (worst case $-0.437$) and on $52\%$--$54\%$ at $p=20$--$30\%$, while the uncertified effort-minimizing \ac{QP} fails on $48\%$--$64\%$. The robust formulation has zero violations in all $150$ runs (Wilson $95\%$ upper bound $7\%$ per level), at the price of Corollary~\ref{cor:price}: floors descend from $-11.06$ to $-12.70$, $-14.47$, $-16.39$ (retained volume $44\%$, $18\%$, $7\%$).

With the metric bound alone (lower block, $\kappa=0.9$) it fails on $6\%$ and $48\%$ of plants at $p=20\%$ and $30\%$; the closed-form drift restores zero violations and cuts the wrench error $3.3\times$.

\begin{table}[tb]\centering\small
    \caption{Robustness campaign ($50$ plants per level, paired across controllers). ``viol.'': fraction below the certified floor. Lower: aggressive floor ($\kappa=0.9$), isolating the drift term of Theorem~\ref{thm:robust}.}
    \label{tab:robust}
    \vspace{-0.5em}
    \renewcommand{\arraystretch}{0.7}
    \setlength{\tabcolsep}{3pt}
    \begin{adjustbox}{max width=0.98\columnwidth}
    \begin{tabular}{clccccc}
    \toprule
    $p$ & controller & mean & worst & viol. & wErr & $\lbar$\\
    \midrule
    $0\%$ & nominal barrier & $+0.164$ & $+0.164$ & $0.00$ & $0.0003$ & $-11.06$\\
    \midrule
    $10\%$ & effort-min QP & $+0.035$ & $-0.512$ & $0.48$ & $0.0064$ & $-11.06$\\
     & nominal barrier & $+0.083$ & $-0.437$ & $0.46$ & $0.0075$ & $-11.06$\\
     & \textbf{robust (Thm.~\ref{thm:robust})} & $+1.693$ & $+1.144$ & $0.00$ & $0.0178$ & $-12.70$\\
    \midrule
    $20\%$ & effort-min QP & $-0.149$ & $-1.245$ & $0.64$ & $0.0066$ & $-11.06$\\
     & nominal barrier & $+0.037$ & $-0.942$ & $0.52$ & $0.0136$ & $-11.06$\\
     & \textbf{robust} & $+3.293$ & $+2.188$ & $0.00$ & $0.0317$ & $-14.47$\\
    \midrule
    $30\%$ & effort-min QP & $-0.389$ & $-2.048$ & $0.64$ & $0.0088$ & $-11.06$\\
     & nominal barrier & $-0.034$ & $-1.550$ & $0.54$ & $0.0426$ & $-11.06$\\
     & \textbf{robust} & $+5.033$ & $+3.327$ & $0.00$ & $0.0420$ & $-16.39$\\
    \midrule
    \multicolumn{7}{l}{\emph{$\kappa{=}0.9$ ablation: metric bound only vs.\ full robust filter}}\\
    $10\%$ & metric bound only & $+0.114$ & $+0.083$ & $0.00$ & $0.0204$ & $-12.57$\\
     & \textbf{robust} & $+0.131$ & $+0.091$ & $0.00$ & $0.0072$ & $-12.57$\\
    $20\%$ & metric bound only & $+0.080$ & $\mathbf{-0.057}$ & $\mathbf{0.06}$ & $0.0314$ & $-14.33$\\
     & \textbf{robust} & $+0.154$ & $+0.098$ & $0.00$ & $0.0076$ & $-14.33$\\
    $30\%$ & metric bound only & $-0.029$ & $\mathbf{-0.413}$ & $\mathbf{0.48}$ & $0.0416$ & $-16.26$\\
     & \textbf{robust} & $+0.168$ & $+0.104$ & $0.00$ & $0.0126$ & $-16.26$\\
    \bottomrule
    \end{tabular}
    \end{adjustbox}
\end{table}



\subsubsection{High-fidelity simulation with measured component data}
\label{ssec:hw}

To expose the certificate to the effects the theory omits, Algorithm~\ref{alg:filter} runs on a 6-DoF hexarotor ($1.27$--$1.8$\,kg) under an SE(3) controller, simulated from measured component data. The propeller is the APC $10{\times}4.7$SF, with $C_T$ and $C_P$ versus advance ratio and rpm from the UIUC wind-tunnel database\footnote{\url{https://m-selig.ae.illinois.edu/props/propDB.html}}; the motor is a SunnySky X2212 $980$\,KV (specified resistance, no-load current, and $15$\,A limit\footnote{\url{https://www.flyingtech.co.uk/product/sunnysky-x2212-13-980kv-ii-2-4s-multirotor-brushless-motor/}}) with back-EMF current limit, Coulomb friction, a $2$\,ms \ac{ESC} current loop, and $3$S battery sag. The plant also includes rotor reaction and gyroscopic torques, H-force, body drag, Dryden turbulence ($1$\,m/s), and speed telemetry with $1$\,ms delay, $0.5\%$ noise, and $10$\,rpm quantization. The filter runs at $500$\,Hz on the quadratic model identified from the static data at $4000$\,rpm ($T_0=0.32$\,s, $v^{\mathrm{sat}}=7330$\,rpm), so all mismatch is physical, not injected. The true margin $h_{\mathrm{true}}$ uses the true thrust slopes and the voltage-limited capacity. Table~\ref{tab:hifi} reports three mission families, five turbulence seeds each: agile roll reversals ($30$--$60^\circ$ at $1$--$1.25$\,Hz), a $1$\,Hz external moment of $0.6$--$1.3$\,N\,m (slung load or manipulator) at two masses, and cruise accelerated to $0$--$15$\,m/s with a slalom.

\begin{table}[tb]\centering\small
    \caption{High-fidelity campaign. Per cell: violation rate ($\min_t h_{\mathrm{true}}<0$)\,/\,lowest $h_{\mathrm{true}}$; second line: \ac{RMS} attitude error. $^{\dagger k}$: $k$ runs not certifiable at $p=0.15$, excluded. Bold: floor violated.}
    \label{tab:hifi}
    \vspace{-0.5em}
    \renewcommand{\arraystretch}{1.0}
    \setlength{\tabcolsep}{2pt}
    \begin{adjustbox}{max width=0.98\columnwidth}
    \begin{tabular}{l|ccccc}
    \toprule
    mission (runs) & effort-min & greedy & low-pass & Alg.~\ref{alg:filter} & robust\\
    \midrule
    agile flight ($20$) & $0.00$\,/\,$+0.84$ & $\mathbf{0.25}$\,/\,$-2.56$ & $\mathbf{0.25}$\,/\,$-2.71$ & $0.00$\,/\,$+0.85$ & $0.00$\,/\,$+3.60$\\
    & $9^\circ$ & $7^\circ$ & $25^\circ$ & $9^\circ$ & $7^\circ$$^{\dagger10}$\\
    external moment ($25$) & $0.00$\,/\,$+0.18$ & $\mathbf{0.76}$\,/\,$-25.01$ & $\mathbf{1.00}$\,/\,$-42.34$ & $0.00$\,/\,$+0.24$ & $\mathbf{0.20}$\,/\,$-5.07$\\
    & $15^\circ$ & $32^\circ$ & $40^\circ$ & $15^\circ$ & $17^\circ$$^{\dagger5}$\\
    cruise $0$--$15$\,m/s ($20$) & $\mathbf{0.25}$\,/\,$-0.71$ & $\mathbf{0.25}$\,/\,$-0.70$ & $\mathbf{0.25}$\,/\,$-0.74$ & $\mathbf{0.25}$\,/\,$-0.09$ & $0.00$\,/\,$+4.00$\\
    & $5^\circ$ & $5^\circ$ & $9^\circ$ & $9^\circ$ & $4^\circ$$^{\dagger10}$\\
    \bottomrule
    \end{tabular}
    \end{adjustbox}
\end{table}

In all $45$ agile and external-moment runs the nominal filter kept $h_{\mathrm{true}}\ge0$, with attitude tracking matching effort-min. Greedy and low-pass DAAM violated the floor in $76$--$100\%$ of external-moment runs, crossing zero spin $9$--$17$ times per run and losing the vehicle (\ac{RMS} altitude error $8$--$15$\,m).

It also exposes three limits. \emph{High airspeed:} at $15$\,m/s, inflow changes the drag and thrust coefficients beyond the model and every allocator falls below the floor; the filter stays closest (lowest $h_{\mathrm{true}}=-0.09$ vs.\ $-0.7$). The robust variant, whose $p=0.15$ covers the $\pm9\%$ rpm dependence of the static drag coefficient, certifies no mission above $5$\,m/s; this regime calls for the airspeed-scheduled barrier of Section~\ref{ssec:ff}. \emph{Cost of robustness:} the $15\%$ torque reserve of the robust variant cannot reject the $1.0$\,N\,m moment at $1.8$\,kg; the attitude loop diverges in $4$ of $5$ runs, and the resulting inflow drives $h_{\mathrm{true}}$ below the floor even though the worst-case margin $h_p$ never falls. \emph{Conservative delay certificate:} with the identified $T_0$, $\tau_d^\star=0.63$\,ms, yet transport delays of $5$--$20$\,ms leave $h_{\mathrm{true}}\ge0$ in every filter run, whereas effort-min fails at $20$\,ms; the filter first fails at $40$\,ms ($1$ of $3$ runs). OSQP solves the filter \ac{QP} in $75$\,$\mu$s on average ($99$th percentile $131$\,$\mu$s) on one $2.1$\,GHz core, well within the $2$\,ms control period.



\section{Conclusion}
\label{sec:conclusion}

Treating control authority as a certified, forward-invariant quantity yields an allocation filter that guards against both collapse modes, with exact confinement, an alignment ratio for exact tracking, and an $\mathcal{O}(\rho^{-1/2})$ bound otherwise, all from $\lbar>\ldrop$; robust invariance costs one $\mathcal{O}(n)$ pass.

\emph{Scope and limitations.} The certificates use the quasi-static rotor model; forward flight enters through bounded drag variation (Theorem~\ref{thm:robust}) and a slope shift (Corollary~\ref{cor:slope}), but the mission itself can become uncertifiable as airspeed grows (Proposition~\ref{prop:achiev}). Confinement to one orthant excludes maneuvers needing rotor reversal (e.g., inverted flight), which require a temporarily lowered floor. Typical \ac{ESC} delays exceed the worst-case ceiling $\tau_d^\star\approx2\times10^{-3}T_0$ and are tolerated, not certified. All evidence is simulation-based, with a quasi-steady rotor model (no unsteady or rotor--rotor aerodynamics). Flight validation, a trajectory-wise delay bound, and online adaptation of $\mathcal{Z}_p$ and of the airspeed-scheduled barrier are the next steps.



\bibliographystyle{IEEEtran}
\bibliography{references}

\end{document}

%% file: figures/tikz/fig_ellipse.tex
\begin{tikzpicture}
    \begin{axis}[
      width=8.5cm, 
      height=7.5cm, 
      axis equal image,
      xmin=-0.45, 
      xmax=0.45, 
      ymin=-0.45, 
      ymax=0.45,
      xlabel={Roll Wrench Rate $\dot M_x$},
      ylabel={Pitch Wrench Rate $\dot M_y$},
      ylabel style={yshift=-0.355cm, xshift=0cm}, 
      xlabel style={yshift=0.105cm, xshift=0cm}, 
      xtick={-0.3, 0, 0.3}, 
      ytick={-0.3, 0, 0.3},
      grid=both, 
      grid style={dashed, gray!40},
      legend style={at={(0.98,0.02)}, anchor=south east, font=\scriptsize, legend cell align=left, fill=white, fill opacity=0.9, text opacity=1}
    ]
    
    \addplot[fill=cBlue!20, draw=cBlue, line width=1pt] table[x=x,y=y]{data/ell_sweet.dat};
    
    \addplot[fill=cOrange!20, draw=cOrange, line width=1pt, dashed] table[x=x,y=y]{data/ell_stopped.dat};
    
    \addplot[fill=cRed!30, draw=cRed, line width=1pt, densely dotted] table[x=x,y=y]{data/ell_sat.dat};
    
    \draw[-{Latex[length=2mm]}, black, line width=1.2pt] (axis cs:0,0) -- (axis cs:0.20,0.145) node[above left, font=\scriptsize] {$\dot w_{\mathrm{des}}$};
    
    \node[font=\scriptsize, cOrange!90!black, align=center] (aniso) at (axis cs:-0.27, 0.38) {\textbf{Anisotropic collapse}\\Loss of specific\\directional authority};
    \draw[-{Latex[length=1.5mm]}, cOrange!90!black, shorten >=2pt, line width=0.6pt] (aniso) -- (axis cs:-0.205, 0.185); 
    
    \node[font=\scriptsize, cRed!90!black, align=center] (iso) at (axis cs:0.30, -0.38) {\textbf{Isotropic collapse}\\Severe overall\\headroom loss};
    \draw[-{Latex[length=1.5mm]}, cRed!90!black, shorten >=2pt, line width=0.6pt] (iso) -- (axis cs:0.03, -0.045); 

    \node[font=\scriptsize, cBlue!90!black, align=center] (sweet) at (axis cs:0.32, 0.38){\textbf{Sweet spot}\\Maximum\\authority}; 
    \draw[-{Latex[length=1.5mm]}, cBlue!90!black, shorten >=2pt, line width=0.6pt] (sweet) -- (axis cs:0.235, 0.235);
    
    \end{axis}
\end{tikzpicture}

%% file: figures/tikz/fig_nested.tex
\begin{tikzpicture}[
        font=\scriptsize, 
        line width=0.5pt,
        halo/.style={fill=white, fill opacity=0.8, text opacity=1, inner sep=1.5pt}
    ]
    \begin{axis}[
        width=6.6cm, height=6.6cm,
        xmin=-1.12, xmax=1.12, ymin=-1.12, ymax=1.12,
        xlabel={$v_1/v_1^{\mathrm{sat}}$}, ylabel={$v_2/v_2^{\mathrm{sat}}$},
        xlabel style={yshift=1.5ex},   
        ylabel style={yshift=-2.0ex},  
        xtick={-1,-0.577,0,0.577,1}, xticklabels={$-1$,$-v^\star$,$0$,$v^\star$,$1$},
        ytick={-1,-0.577,0,0.577,1}, yticklabels={$-1$,$-v^\star$,$0$,$v^\star$,$1$},
        tick label style={font=\tiny}, axis equal image,
        clip mode=individual
    ]
    
    \draw[dashed,cRed!80,line width=0.8pt] (axis cs:-1,-1) rectangle (axis cs:1,1);
    
    \draw[cOrange,line width=1pt] (axis cs:-1.12,0)--(axis cs:1.12,0);
    \draw[cOrange,line width=1pt] (axis cs:0,-1.12)--(axis cs:0,1.12);
    
    \addplot[cGrey!45,line width=0.6pt,fill=cGrey!14] table {data/n2_Keps.dat};
    \addplot[cBlue,line width=0.8pt,fill=cBlue!28] table {data/n2_C.dat};
    
    \node[cGrey!70!black,font=\tiny,anchor=center] at (axis cs:-0.83,0.21) {$K_\varepsilon$};
    \node[cBlue,font=\scriptsize,anchor=center] at (axis cs:0.577,0.577) {$\Cset$};
    \node[cBlue,font=\tiny,anchor=center] at (axis cs:-0.577,0.577) {$\Cset$};
    \node[cBlue,font=\tiny,anchor=center] at (axis cs:-0.577,-0.577) {$\Cset$};
    \node[cBlue,font=\tiny,anchor=center] at (axis cs:0.577,-0.577) {$\Cset$};
    
    \draw[{Latex[length=1.2mm]}-{Latex[length=1.2mm]},cOrange!80!black,line width=0.6pt] 
      (axis cs:0.0,0.80)--(axis cs:0.28,0.80);
    
    \draw[{Latex[length=1.2mm]}-{Latex[length=1.2mm]},cRed,line width=0.6pt] 
      (axis cs:0.90,0.34)--(axis cs:1.0,0.34);
    
    \draw[-{Latex[length=1.5mm]},cBlue,line width=0.9pt] (axis cs:0.577,0.86)--(axis cs:0.577,1.02);
    \node[cBlue,font=\tiny,anchor=south] at (axis cs:0.577,0.98) {$\nabla h$};
    
    \draw[-{Latex[length=2mm]},cRed,dashed,line width=1pt] 
      (axis cs:0.40,0.30)..controls (axis cs:0.05,0.0)..(axis cs:0.40,-0.30);
    \node[cRed,font=\tiny,anchor=north west,align=left] at (axis cs:0.32,-0.25)
      {$\Delta_k$};
    
    \end{axis}
\end{tikzpicture}

%% file: figures/tikz/fig_sweep.tex
\pgfplotsset{sweepaxis/.style={
  width=\panelwd, height=\panelht,
  xmin=0.58, xmax=1.28,
  xlabel={collective ($\times$ sweet-spot value)},
  xtick={0.6,0.8,1.0,1.2},
  every axis title/.style={font=\small\bfseries},
  legend style={font=\small, fill=white, fill opacity=0.85, text opacity=1}
}}
\begin{tikzpicture}
  \begin{axis}[sweepaxis,
    ymin=-1.32, ymax=0.62,
    ylabel={$\min_t h$ (solid), $\overline{h}$ (dashed)},
    ytick={-1.2,-0.8,-0.4,0,0.4},
    legend style={at={(0.50,1.10)}, anchor=south east, anchor=north, legend cell align=left, legend columns=-1}
  ]
  \addplot[draw=none, fill=cRed!12, forget plot]
    coordinates {(0.58,-1.32)(1.28,-1.32)(1.28,0)(0.58,0)};
  \addplot[draw=none, fill=cGrey!22, forget plot]
    coordinates {(0.58,-1.32)(0.65,-1.32)(0.65,0.62)(0.58,0.62)};
  \draw[cRed, line width=0.8pt] (axis cs:0.58,0) -- (axis cs:1.28,0);

  \addplot[cBlue, mark=*, mark size=1.6pt, line width=1.2pt]
    table[x=thrust,y=energy_hmin]{data/sweep.dat};
    \addlegendentry{effort-min}
  \addplot[cOrange, mark=square*, mark size=1.6pt, line width=1.2pt]
    table[x=thrust,y=greedy_hmin]{data/sweep.dat};
    \addlegendentry{greedy DAAM}
  \addplot[cGrey!80, mark=triangle*, mark size=1.6pt, line width=1.2pt]
    table[x=thrust,y=lowpass_hmin]{data/sweep.dat};
    \addlegendentry{low-pass DAAM}
  \addplot[cGreen!70!black, mark=diamond*, mark size=1.8pt, line width=0.9pt]
    table[x=thrust,y=cbf_hmin]{data/sweep.dat};
    \addlegendentry{proposed}
  \addplot[cBlue, dashed, line width=0.7pt, forget plot] table[x=thrust,y=energy_hmean]{data/sweep.dat};
  \addplot[cOrange, dashed, line width=0.7pt, forget plot] table[x=thrust,y=greedy_hmean]{data/sweep.dat};
  \addplot[cGrey!80, dashed, line width=0.7pt, forget plot] table[x=thrust,y=lowpass_hmean]{data/sweep.dat};
  \addplot[cGreen!70!black, dashed, line width=0.7pt, forget plot] table[x=thrust,y=cbf_hmean]{data/sweep.dat};

  \node[font=\small, cRed, anchor=south west]
    at (axis cs:0.66,-0.10) {certified floor};
  \node[font=\small, cGrey!70!black, rotate=90, anchor=south, align=center]
    at (axis cs:0.635,-0.65) {$L^{\mathrm{op}}<\ldrop$};
  \end{axis}
\end{tikzpicture}%
\hspace{0.35cm}%
\begin{tikzpicture}
  \begin{axis}[sweepaxis,
    ymode=log,                     
    ymin=1.5e-4, ymax=4e-1,
    ylabel={RMS wrench error},
    ytick={1e-3,1e-2,1e-1},
    legend style={at={(0.50,1.10)}, anchor=south east, anchor=north, legend cell align=left, legend columns=-1}
  ]
  \addplot[draw=none, fill=cGrey!22, forget plot]
    coordinates {(0.58,1.5e-4)(0.65,1.5e-4)(0.65,4e-1)(0.58,4e-1)};

  \addplot[cBlue, mark=*, mark size=1.6pt, line width=1.2pt]
    table[x=thrust,y=energy_werr]{data/sweep.dat};
    \addlegendentry{effort-min}
  \addplot[cOrange, mark=square*, mark size=1.6pt, line width=1.2pt]
    table[x=thrust,y=greedy_werr]{data/sweep.dat};
    \addlegendentry{greedy DAAM}
  \addplot[cGrey!80, mark=triangle*, mark size=1.6pt, line width=1.2pt]
    table[x=thrust,y=lowpass_werr]{data/sweep.dat};
    \addlegendentry{low-pass DAAM}
  \addplot[cGreen!70!black, mark=diamond*, mark size=1.8pt, line width=0.9pt]
    table[x=thrust,y=cbf_werr]{data/sweep.dat};
    \addlegendentry{proposed}

  \node[font=\small, cGrey!70!black, rotate=90, anchor=south, align=center]
  at (axis cs:0.635, 15e-4) {$L^{\mathrm{op}}<\ldrop$};
  \end{axis}
\end{tikzpicture}

%% file: figures/tikz/fig_lpf2.tex
\begin{tikzpicture}
  \begin{axis}[
    width=\panelwd, height=\panelht,
    xmin=0, xmax=2,
    xlabel={time $t/T_0$},
    ymode=log,
    ymin=2e-3, ymax=3e3,
    ylabel={$\max_i|\dot v_i^{\mathrm{cmd}}|$ $[v^{\mathrm{sat}}/T_0]$},
    ytick={1e-2,1e0,1e2},
    every axis title/.style={font=\small\bfseries},
    legend style={font=\small, fill=white, fill opacity=0.85, text opacity=1,
                  at={(0.50,1.10)}, anchor=north, legend cell align=left, legend columns=-1}
  ]
  \addplot[cBlue, line width=1.2pt, dashed]
    table[x=t,y=rate_energy]{data/rate_all.dat};
    \addlegendentry{effort-min}

  \addplot[cOrange, line width=1.2pt]
    table[x=t,y=rate_greedy]{data/rate_all.dat};
    \addlegendentry{greedy DAAM}

  \addplot[cGrey!80, line width=1.2pt]
    table[x=t,y=rate_lowpass]{data/rate_all.dat};
    \addlegendentry{low-pass DAAM}

  \addplot[cGreen!70!black, line width=1.4pt]
    table[x=t,y=rate_cbf]{data/rate_all.dat};
    \addlegendentry{proposed}
  \end{axis}
\end{tikzpicture}%
\hspace{0.35cm}%
\begin{tikzpicture}
  \begin{axis}[
    width=\panelwd, height=\panelht,
    xmin=0, xmax=2,
    xlabel={time $t/T_0$},
    ymin=-0.72, ymax=0.82,
    ylabel={$h(t)=L(v(t))-\lbar$},
    ytick={-0.5,-0.25,0,0.25,0.5,0.75},
    every axis title/.style={font=\small\bfseries},
    legend style={font=\small, fill=white, fill opacity=0.85, text opacity=1,
                  at={(0.50,1.10)}, anchor=north, legend cell align=left, legend columns=-1}
  ]
  \addplot[draw=none, fill=cRed!12, forget plot]
    coordinates {(0,-0.72)(2,-0.72)(2,0)(0,0)};
  \draw[cRed, line width=0.8pt]
    (axis cs:0,0) -- (axis cs:2,0);

  \addplot[cBlue, line width=1.2pt, dashed]
    table[x=t,y=energy]{data/h_all.dat};
    \addlegendentry{effort-min}

  \addplot[cOrange, line width=1.2pt]
    table[x=t,y=greedy]{data/h_all.dat};
    \addlegendentry{greedy DAAM}

  \addplot[cGrey!80, line width=1.2pt]
    table[x=t,y=lowpass]{data/h_all.dat};
    \addlegendentry{low-pass DAAM}

  \addplot[cGreen!70!black, line width=1.0pt]
    table[x=t,y=cbf]{data/h_all.dat};
    \addlegendentry{proposed}

  \node[font=\small, cRed, anchor=north west]
    at (axis cs:0.02,-0.01) {certified floor};



  \end{axis}
\end{tikzpicture}

%% file: figures/tikz/fig_applic.tex
\begin{tikzpicture}
  \begin{axis}[
    width=\panelwd, height=\panelht,
    xmin=0.525, xmax=1.325, ymin=0.375, ymax=3.125,
    xlabel={collective ($\times$ sweet-spot value)}, ylabel={moment amplitude ($\times$ ref)},
    xtick={0.6,0.8,1.0,1.2}, ytick={0.5,1,1.5,2,2.5,3},
    colormap={rb}{rgb255(0cm)=(214,39,40); rgb255(1cm)=(255,255,255); rgb255(2.6cm)=(31,119,180)},
    point meta min=-1.6, point meta max=1.1,
    colorbar, colorbar style={width=0.25cm, ytick={-1.5,-1,-0.5,0,0.5,1}, ylabel={$L^{\mathrm{op}}-\ldrop$ [nats]}, ylabel style={font=\small}},
    axis on top, grid=none, enlargelimits=false
  ]
  \addplot[matrix plot*, mesh/rows=11, mesh/cols=16, mesh/ordering=y varies, point meta=explicit]
    table[x=thrust, y=amp, meta=window]{data/certmap.dat};
  \addplot[white, line width=1.6pt] table[x=thrust,y=amp]{data/certfront.dat};
  \addplot[black, line width=0.6pt, densely dashed] table[x=thrust,y=amp]{data/certfront.dat};
  \draw[black, line width=0.6pt] (axis cs:0.7,1.0) rectangle (axis cs:1.2,2.0);
  \draw[black, line width=0.6pt, dotted] (axis cs:0.68,1.0) rectangle (axis cs:0.82,2.0);
  \node[font=\small, anchor=north east, fill=white, fill opacity=0.8, text opacity=1, inner sep=1pt] at (axis cs:1.30,3.07) {not certifiable};
  \node[font=\small, anchor=south east, fill=white, fill opacity=0.8, text opacity=1, inner sep=1pt] at (axis cs:1.30,0.45) {certifiable};
  \end{axis}
\end{tikzpicture}%
\hspace{0.35cm}%
\begin{tikzpicture}
  \begin{axis}[
    width=\panelwd, height=\panelht,
    xmin=-0.5, xmax=15.5, ymin=-1.2, ymax=0.5,
    xlabel={airspeed $V_\infty$ [m/s]}, ylabel={$\min_t h_s$ (true readiness margin)},
    xtick={0,3,6,9,12,15},
    legend style={font=\small, fill=white, fill opacity=0.85, text opacity=1,
                  at={(0.02,0.03)}, anchor=south west, legend cell align=left}
  ]
  \addplot[draw=none, fill=cRed!12, forget plot] coordinates {(-0.5,-1.2)(15.5,-1.2)(15.5,0)(-0.5,0)};
  \draw[cRed, line width=0.8pt] (axis cs:-0.5,0) -- (axis cs:15.5,0);
  \addplot[draw=none, fill=cGrey!22, forget plot] coordinates {(9.33,-1.2)(15.5,-1.2)(15.5,0.5)(9.33,0.5)};
  \addplot[cBlue, mark=*, mark size=1.6pt, line width=1.2pt] table[x=V,y=energy]{data/ff.dat};
  \addlegendentry{effort-min}
  \addplot[cOrange, mark=square*, mark size=1.6pt, line width=1.2pt] table[x=V,y=greedy]{data/ff.dat};
  \addlegendentry{greedy DAAM}
  \addplot[cGreen!70!black, mark=diamond*, mark size=1.8pt, line width=0.9pt] table[x=V,y=cbf]{data/ff.dat};
  \addlegendentry{proposed (hover model)}
  \addplot[cGreen!40!black, dashed, mark=triangle*, mark size=1.8pt, line width=0.9pt] table[x=V,y=aware]{data/ff.dat};
  \addlegendentry{proposed (airspeed-scheduled)}
  \node[font=\small, cGrey!70!black, anchor=north, align=center] at (axis cs:12.4,0.5) {$L^{\mathrm{op}}<\lbar$};
  \end{axis}
\end{tikzpicture}

%% file: figures/tikz/fig_delayregimes.tex
\begin{tikzpicture}
    \begin{axis}[
      width=\panelwd, height=\panelht,
      xmode=log, xmin=0.2, xmax=250, ymin=-1.2, ymax=0.40,
      xlabel={input delay $\tau_d$ [$10^{-3}\,T_0$]}, ylabel={$\min_t h$ (w.r.t.\ $\lbar$)},
      xtick={0.5,1,2,5,10,20,50,100,200},
      xticklabels={0.5,1,2,5,10,20,50,100,200},
      ytick={-1,-0.75,-0.5,-0.25,0,0.25},
      legend style={at={(0.02,0.03)}, anchor=south west, font=\small, fill=white, fill opacity=0.9, text opacity=1, legend cell align=left},
      clip mode=individual
    ]
    \addplot[draw=none, fill=cGreen!16, forget plot]
      coordinates {(0.2,-1.2)(1.9656,-1.2)(1.9656,0.40)(0.2,0.40)}--cycle;
    \addplot[draw=none, fill=cOrange!13, forget plot]
      coordinates {(1.9656,-1.2)(91.6,-1.2)(91.6,0.40)(1.9656,0.40)}--cycle;
    \addplot[draw=none, fill=cRed!14, forget plot]
      coordinates {(91.6,-1.2)(250,-1.2)(250,0.40)(91.6,0.40)}--cycle;
    \node[font=\small, cGreen!45!black, anchor=north] at (axis cs:0.63,0.385) {certified};
    \node[font=\small, cOrange!80!black, anchor=north] at (axis cs:14,0.385) {tolerated (not certified)};
    \node[font=\small, cRed, anchor=north] at (axis cs:155,0.385) {fails};
    \draw[cRed, line width=0.8pt] (axis cs:0.2,0) -- (axis cs:250,0);
    \addplot[cGrey, densely dotted, line width=0.9pt, domain=0.2:250, samples=2] {0.1636};
    \addlegendentry{delay-free ($\tau_d=0$)}
    \addplot[cRed, mark=*, mark size=1.5pt, line width=0.9pt]
      table[x=tau_d_ms,y=minh_plain]{data/delay.dat};
    \addlegendentry{plain floor $\lbar$}
    \addplot[cBlue, mark=square*, mark size=1.7pt, line width=0.9pt]
      table[x=tau_d_ms,y=minh_margined]{data/delay_marg.dat};
    \addlegendentry{margined floor $\lbar+\eta(\tau_d)$}
    \addplot[black, dashed, line width=0.8pt]
      table[x=tau_d_ms,y=eta]{data/delay_eta.dat};
    \addlegendentry{bound $-\eta(\tau_d)$, Prop.~\ref{prop:delay}}
    \draw[black, dotted, line width=0.8pt] (axis cs:1.9656,-1.2)--(axis cs:1.9656,0.40);
    \node[font=\small, anchor=south west, rotate=90] at (axis cs:2.25,-1.15) {$\tau_d^\star$};
    \end{axis}
\end{tikzpicture}